\documentclass[10pt,twocolumn,letterpaper]{article}

\usepackage[pagenumbers]{wacv}      % To produce the REVIEW version for the algorithms track
\usepackage{graphicx}
\usepackage{wrapfig}
\usepackage{stfloats}
\usepackage{amssymb}
\usepackage{xcolor}
\usepackage{algorithm}
\usepackage{algpseudocode}
\usepackage{float} 
\usepackage{verbatim}
\usepackage{tcolorbox}
\tcbuselibrary{skins}
\usepackage{algorithm,algpseudocode}

\newtcolorbox{algobox}[1][]{
  enhanced,
  sharp corners,
  colback=white,
  colframe=black,
  fonttitle=\bfseries,
  top=1pt, bottom=1pt, left=3pt, right=3pt,
  attach boxed title to top left={
    yshift=-2mm, xshift=2mm
  },
  boxed title style={
    size=small,
    colback=white,
    colframe=black,
    sharp corners
  },
  title=#1
}

\usepackage{graphicx}
\usepackage{amsmath}
\usepackage{amssymb}
\usepackage{booktabs}
\usepackage{amsthm}

\newtheorem{theoremOne}{Theorem}[section]
\newtheorem{theoremTwo}{Theorem}[section]

\newcommand{\cmark}{\ding{51}}
\newcommand{\xmark}{\ding{55}}

\usepackage{multirow}
\usepackage{soul}
\usepackage{pifont}
\usepackage{marvosym}

\usepackage{amsfonts,amssymb}

\usepackage[dvipsnames]{xcolor}
\usepackage{colortbl}
\definecolor{mygray}{gray}{.9}

\newcommand{\ours}{N\textsc{erve}}

\newcommand{\phz}{\phantom{0}} 

\newcommand{\mr}[1]{\mathrm{#1}}

\usepackage{graphics}
\usepackage{xcolor}          % provides \textcolor

\newcommand{\mypar}[1]
{\vspace{3pt}\noindent\textbf{#1~}}

\definecolor{cliptta_color}{rgb}{1, 0.97, 0.92}
\definecolor{lightgray}{rgb}{0.95, 0.95, 0.95}
\definecolor{BrickRed}{rgb}{0.8, 0.25, 0.33}
\definecolor{OliveGreen}{rgb}{0.33, 0.42, 0.18}
\definecolor{DarkGreen}{RGB}{0,100,0}
\definecolor{DarkRed}{RGB}{139,0,0}
\usepackage[normalem]{ulem}

\newcommand{\better}[1]{\textcolor{DarkGreen}{\scriptsize{\,\,($\uparrow$#1)}}}
\newcommand{\worse}[1]{\textcolor{DarkRed}{\scriptsize{\,\,\,($\downarrow$#1)}}}

\newcommand{\rcol}{\rowcolor{orange!12}}

\algnewcommand{\LineComment}[1]{\Statex \textcolor{gray}{// \emph{#1}}}
\usepackage{pgfplots}
\usepackage{pgfplotstable}
\usepackage{amsmath}
\usepackage{sansmath}
\pgfplotsset{compat=1.18}
\usetikzlibrary{positioning, decorations.pathreplacing}

\usepackage{xcolor}          % for \textcolor
\usepackage[utf8]{inputenc}  % if you really need Unicode
\usepackage{float}

\newcommand{\mathbbm}[1]{\text{\usefont{U}{bbm}{m}{n}#1}} 

\newcommand{\tr}[1]{#1^{\top}}

\newcommand{\One}{\mathbbm{1}}

\usepackage{graphicx,scalerel}

\newcommand\wt[1]{\hstretch{1.25}{\tilde{\hstretch{.8}{#1}}}}

\definecolor{wacvblue}{rgb}{0.21,0.49,0.74}
\usepackage[pagebackref,breaklinks,colorlinks,allcolors=wacvblue]{hyperref}

\usepackage[capitalize]{cleveref}
\crefname{section}{Sec.}{Secs.}
\Crefname{section}{Section}{Sections}
\Crefname{table}{Table}{Tables}
\crefname{table}{Tab.}{Tabs.}

\def\wacvPaperID{1486} % *** Enter the WACV Paper ID here
\def\confName{WACV}
\def\confYear{2026}

\title{\ours{}: \underline{N}eighbourhood \& \underline{E}ntropy-guided \underline{R}andom-walk for training free open-\underline{V}ocabulary s\underline{E}gmentation}

\author{Kunal Mahatha \textsuperscript{\Letter} \quad 
Jose Dolz \quad Christian Desrosiers \\[1em]
LIVIA, ÉTS Montréal, Canada \\
International Laboratory on Learning Systems (ILLS), \\
McGILL - ETS - MILA - CNRS - Université Paris-Saclay - CentraleSupélec, Canada \\
\Letter \ \tt \small kunal.mahatha.1@ens.etsmtl.ca\\[0.5em]}

\begin{document}
\maketitle
\begin{abstract}

\noindent Despite recent advances in Open-Vocabulary Semantic Segmentation (OVSS), existing training-free methods face several limitations: use of computationally expensive affinity refinement strategies, ineffective fusion of transformer attention maps due to equal weighting or reliance on fixed-size Gaussian kernels to reinforce local spatial smoothness, enforcing isotropic neighborhoods. We propose a strong baseline for training-free OVSS termed as NERVE (Neighbourhood \& Entropy-guided Random-walk for open-Vocabulary sEgmentation), which uniquely integrates global and fine-grained local information, exploiting the neighbourhood structure from the self-attention layer of a stable diffusion model. We also introduce a stochastic random walk for refining the affinity rather than relying on fixed-size Gaussian kernels for local context. This spatial diffusion process encourages propagation across connected and semantically related areas, enabling it to effectively delineate objects with arbitrary shapes. Whereas most existing approaches treat self-attention maps from different transformer heads or layers equally, our method uses entropy-based uncertainty to select the most relevant maps. Notably, our method does not require any conventional post-processing techniques like Conditional Random Fields (CRF) or Pixel-Adaptive Mask Refinement (PAMR). Experiments are performed on 7 popular semantic segmentation benchmarks, yielding an overall state-of-the-art zero-shot segmentation performance, providing an effective approach to open-vocabulary semantic segmentation.

\end{abstract}
    
\section{Introduction}
\label{sec:intro}

Deep learning models have made significant strides in dense image prediction tasks, particularly in semantic segmentation \cite{minaee2021image}. However, their reliance on a fixed set of predefined classes limits their applicability in numerous real-world scenarios where the categories of interest are often unknown or variable. To address this limitation, open-vocabulary semantic segmentation (OVSS) has recently emerged as a promising alternative \cite{bucher2019zero, zhao2017open, maskclip}. OVSS enables models to recognize and segment novel categories that were not encountered during training, offering greater flexibility and generalization. 
\footnote{Our code is publicly available at: \url{https://github.com/kunal-mahatha/nerve/}}

Most OVSS methods are built on vision-language models, particularly Contrastive Language-Image Pre-training (CLIP) \cite{clip}, which has demonstrated strong zero-shot performance in visual recognition. These methods generally fall into two broad categories: \emph{training-based} \cite{cho2024cat, xu2023side, yu2023convolutions,barsellotti2023enhancing, ghiasi2022scaling, liang2023open, qin2023freeseg, xu2023open, xu2022simple} and \emph{training-free} \cite{clearclip, sclip, maskclip} approaches. Methods in the first category typically include a fully-supervised (or weakly-supervised) training step, where pixel-wise (or image-wise) annotations from a limited set of categories are used to transfer language–vision relationships from image level to the pixel level. However, these methods still depend on extensive annotations for related categories, and their performance can be influenced by the choice of training dataset used for adaptation \cite{naclip}. 

In contrast, training-free methods \cite{barsellotti2024fossil, bousselham2024grounding, corradini2024freeseg, ovdiff, li2023clip, luo2024emergent, shin2022reco, sclip, wysoczanska2024clip, maskclip,clearclip} do not require access to additional data for adaptation, making them suitable for realistic scenarios where large labeled datasets are scarce and novel classes cannot be anticipated. Several methods leverage the visual representations learned by image-level supervised vision transformers (ViTs) such as CLIP. Some approaches modify the self-attention mechanism—for instance, by incorporating query-query or key-key interactions (also known as self-self attention) \cite{bousselham2024grounding, li2023clip, maskclip, sclip, clearclip}. Others directly use the value vectors from CLIP as dense features for pixel-level classification \cite{maskclip}. An alternative line of work employs additional vision encoders that have been pre-trained using supervised \cite{naseer2021intriguing}, self-supervised \cite{he2020momentum}, or unsupervised \cite{simeoni2023unsupervised} learning objectives. More recently, state-of-the-art performance in OVSS %open-vocabulary semantic segmentation (OVSS) 
has been achieved by harnessing visual features from generative models such as Stable Diffusion \cite{barsellotti2024fossil, corradini2024freeseg, ovdiff, xu2023open, diffseg, iseg}. In this paper, we consider the more challenging scenario of training-free OVSS.

Despite the recent advancements in OVSS, current training-free methods for this problem still face notable limitations. First, most of these methods obtain segmentation predictions by combining attention maps from the last ViT layer \cite{iseg,naclip,ding2023open,clearclip,sclip,proxyclip} or from multiple layers with upsampling \cite{khani2023slime,sun2024cliper}, without taking into account their relative usefulness in the final prediction (e.g., simply averaging the maps). Moreover, methods that rely directly on CLIP’s coarse attention maps typically require an additional segmentation refinement step to address their low spatial resolution \cite{sclip,clearclip,naclip}. This limitation can be partially mitigated by leveraging the finer-grained attention maps produced by diffusion models \cite{iseg,diffseg}. Nevertheless, this comes at a substantial computational cost, as it involves operations over a large, dense similarity matrix that models interactions between all pairs of spatial positions. Additionally, since attention maps do not explicitly capture local spatial relationships, existing approaches may struggle to preserve the contextual integrity of individual objects. To address this, \cite{naclip} introduced a simple yet effective solution: augmenting attention maps with 2D isotropic Gaussian kernels centered on each patch to increase the attention weight of neighboring regions. While this improves OVSS performance, it assumes that local object context is circular and fixed in size, an assumption that often does not hold in practice. As a result, a separate refinement step is required in post-processing to improve the segmentation. 

To address these limitations, we introduce \ours{}, a training-free and computationally efficient open-vocabulary semantic segmentation (OVSS) method that exploits neighborhood structure and entropy-guided random walks. \ours{} formulates segmentation as a linear stochastic process, where image patches are treated as graph nodes and edge weights determine the transition probabilities between nodes during a random walk. Rather than relying on fixed-size Gaussian kernels to model local object context \cite{naclip}, our method promotes locality within the random walk by assigning higher transition probabilities to neighboring regions. This spatial diffusion process encourages propagation across connected and semantically related areas, enabling the random walk to effectively delineate objects with arbitrary shapes. \ours{} also captures pairwise relationships across the entire image efficiently by leveraging the low-rank structure of the self-attention matrix. In contrast to current approaches that explicitly reconstruct the full attention matrix—resulting in quadratic scaling with the number of patches \cite{iseg}—our method scales linearly, significantly accelerating the segmentation refinement process. Finally, whereas most existing approaches treat self-attention maps from different transformer heads or layers equally \cite{iseg,naclip,ding2023open,clearclip,sclip,proxyclip,khani2023slime,sun2024cliper}, our method uses entropy-based uncertainty to select the most relevant maps for the random walk.

Our main contributions can be summarized as follows:
\begin{itemize}\setlength\itemsep{-1em}
    \item We propose \ours{}, a novel training-free method for open-vocabulary semantic segmentation (OVSS) based on neighborhood- and entropy-guided random walks. A key innovation of \ours{} is its ability to efficiently capture \emph{both local and global relationships} within an image, enabling faster and more accurate segmentation. By accounting for the uncertainty across self-attention maps during their combination, \ours{} also suppresses noisy relationships that can lead to segmentation errors.\\
    \item We motivate our method using a stochastic process that models an infinite-step random walk on a graph, and show how the resulting segmentation can be computed efficiently through either exact or approximate inference. \\    
    \item We demonstrate the state-of-the-art performance of \ours{} through comprehensive experiments conducted on five benchmark datasets and two variants. Additionally, the contribution of each individual component is validated through extensive ablation studies.
\end{itemize}

%-------------------------------------------------------------------------

\begin{comment}
\begin{figure}[!ht]                % ’t’ means place at top of column
  \centering
  \includegraphics[width=1.0\linewidth]{fig/intro-d.pdf}
  \caption{\textbf{\ours{} Pipeline}: CLIP cross‑attention and Stable Diffusion self‑attention are entropy‑fused, then refined via a stochastic random walk to yield final masks.}
  \label{fig:example_single}
\end{figure}
\end{comment}

\begin{figure*}[t]
    \centering
    \includegraphics[width=.8\textwidth]{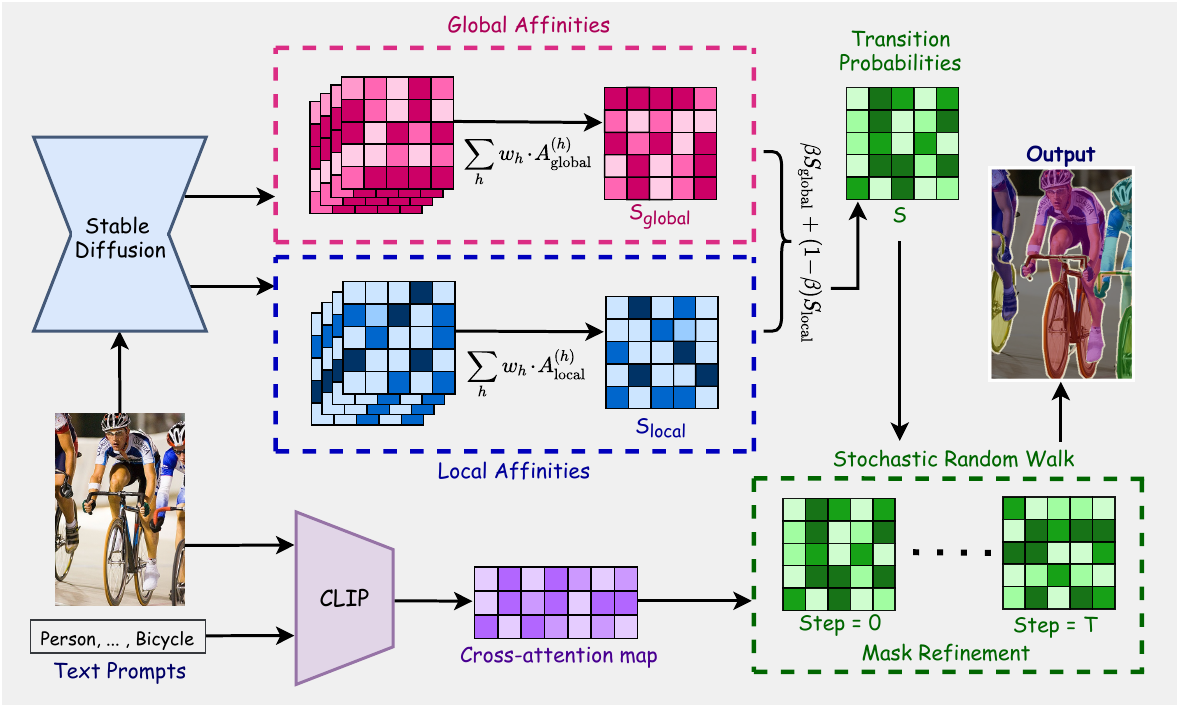}
    \caption{\textbf{Overview of our training-free open-vocabulary segmentation pipeline.} Given an image and a set of text prompts, we extract cross-attention maps using CLIP and self-attention maps from a Stable Diffusion encoder. We compute entropy-guided fusion across attention heads $h$ to obtain global (\(A_{\text{global}}\)) and local (\(A_{\text{local}}\)) affinities. These are normalized and linearly combined into a final stochastic matrix \(S\), which is used in a truncated stochastic random walk to propagate semantic information and generate refined segmentation masks.}
    \label{fig:pipeline}
\end{figure*}

\section{Related Works}
\label{sec:related_works}

Open-vocabulary semantic segmentation (OVSS) has gained significant attention in recent years \cite{clipes, ranasinghe2023perceptual, shin2022reco}. Unlike traditional semantic segmentation, which assumes a fixed set of categories shared between training and testing, OVSS aims to segment objects belonging to arbitrary, potentially unseen categories. Existing methods for this task can be broadly categorized in into training-based \cite{cha2023learning, luo2023segclip, ren2023viewco, xu2023learning} and training-free approaches \cite{lin2024tagclip,sclip,maskclip,clearclip}.

Training-based methods for OVSS can be further categorized into fully-supervised and weakly-supervised approaches. Fully-supervised methods \cite{barsellotti2023enhancing, ghiasi2022scaling, liang2023open, qin2023freeseg, xu2023open, xu2023side, xu2022simple} adapt a pre-trained CLIP model for segmentation by leveraging pixel-level annotations from a fixed set of classes, while aiming to retain its generalization to unseen categories. However, the adaptation datasets used in these approaches often exhibit significant class overlap with the test images, potentially inflating performance on unseen classes. In contrast, weakly supervised approaches adapt the CLIP model using only image-level annotations, typically provided as textual descriptions of the categories present in each image \cite{cai2023mixreorg, cha2023learning, chen2023exploring, liu2022open, wu2023clipself, groupvit, zhang2023uncovering}. Visual–textual alignment is generally achieved through a contrastive loss \cite{groupvit}, mirroring CLIP's original pre-training strategy. This alignment can be further improved by incorporating auxiliary learning techniques, such as online pixel clustering \cite{liu2022open} or additional contrastive objectives between features from clean and corrupted image variants \cite{cai2023mixreorg}.

In contrast to training-based approaches, training-free methods \cite{maskclip,sclip,clearclip,proxyclip,car,naclip} repurpose vision-language models like CLIP—originally designed for classification—for segmentation tasks without additional learning. These methods typically rely on CLIP’s embeddings or adapt its attention maps (or their variants) to produce segmentation outputs. For instance, MaskCLIP \cite{maskclip} ignores the final's layer self-attention maps and instead uses value embeddings of this layer to achieve pixel-level segmentation. Other models like SCLIP \cite{sclip} and ClearCLIP \cite{clearclip} replace standard self-attention maps with query-to-query or key-to-key attention mechanisms, known as self-self attention. While many methods rely on CLIP’s image encoder to extract visual features, alternative approaches instead utilize diffusion models for this purpose \cite{xu2023open,ovdiff,diffseg,iseg}. For instance, DiffSegmenter \cite{diffseg} and iSeg \cite{iseg} exploit self- and cross-attention maps from a pre-trained Stable Diffusion model \cite{stablediffusion} to enable open-vocabulary segmentation.

Our work introduces a novel method for OVSS that extends the neighborhood-aware strategy of NACLIP \cite{naclip} and the iterative refinement approach of iSeg \cite{iseg} in a efficient framework based on random walks. Unlike NACLIP, which uses fixed-size Gaussian kernels to reinforce attention between neighboring regions, our approach increases the transition probabilities toward neighboring regions during the random walk. Additionally, we introduce a non-zero stopping probability that encourages locality in the segmentation process. Compared to the refinement strategy of iSeg, which involves computing and manipulating a large dense matrix of size $N^2$ (with $N\!=\!H\!\times\!W$, where $H$ and $W$ denote the height and width of the feature map), our \ours{} method leverages the low-rank structure of the global attention matrix and the sparsity of the local attention matrix to reduce the computational complexity to be \emph{linear} in $N$. Finally, related methods often aggregate attention maps from different ViT heads or layers, typically using simple averaging \cite{iseg,naclip,ding2023open,clearclip,sclip,proxyclip,khani2023slime,sun2024cliper} or fixed weighting schemes \cite{diffseg}. In contrast, our approach introduces a dynamic weighting mechanism guided by entropy-based uncertainty, enabling more adaptive and informative attention fusion.

\section{Methodology}
\label{sec:methodology}

%We propose a fully training-free framework for open-vocabulary semantic segmentation that unifies global semantic priors from vision-language models with local spatial coherence extracted from attention mechanisms. 

\subsection{Random-Walk Model for OVSS}\label{sec:rw-ovss}

Our training-free OVSS approach, which is summarized in Figure \ref{fig:pipeline}, is based on a stochastic process modeling a random-walk on a graph $\mathcal{G} = (\mathcal{V}, \mathcal{E})$, where each node $i\!\in\!\mathcal{V}$ corresponds to a spatial token and each edge $(i, j)\!\in\!\mathcal{E}$ has a weight $a_{i,j}$ encoding the affinity between nodes $i$ and $j$. The main idea is to diffuse the coarse CLIP-based segmentation probabilities using the fine-grained self-attention maps of a diffusion model. 

Let $N\!=\!|\mathcal{V}|$ be the number of tokens/nodes and as $K$ the number classes in the test set. For each step at an arbitrary node $i$, the random walk can either \emph{i}) move to a neighbor node with probability $\alpha\!<\!1$, or \emph{ii}) stop the walk and generate a label with probability $1\!-\!\alpha$. In the first case, we select the neighbor node $j\!\in\!\mathcal{V}$ with probability
\begin{equation}\label{eq:row-norm}
s_{i \to j} \, := \, \frac{a_{i,j}}{\sum_{j'} a_{i,j'}}.
\end{equation}
In the second case, we generate a label $y_k$ with probability $g_{i \to k}$. Denoting as $S\!\in\!\mathbb{R}^{N \times N}$ the row-stochastic node-to-node transition matrix such that $S_{i,j}\!=\!s_{i \to j}$ and $G\!\in \!\mathbb{R}^{N \times K}$ the row-stochastic node-to-label generation matrix such that $G_{i,k}\!=\!g_{i \to k}$. 
The expected probability of generating label $y_k$ in a infinite-length random walk, starting at node $i$, is given by $[P_{\infty}]_{i,k}$ where
\begin{equation}\label{rw-infinite-sum}
P_{\infty} \, := \, (1\!-\!\alpha) \Big(\sum_{t=0}^{\infty} \alpha^t S^t\Big) G.
\end{equation}
with $S^t$ being obtained by multiplying matrix $S$ with itself $t\!-\!1$ times. 

While computing this infinite sum directly is impossible, matrix $P_{\infty}$ can be obtained using a matrix inversion as described in the following theorem.
\begin{theoremOne}%[Pythagorean theorem]
\label{theorem1}
The expected  label probability matrix $P_{\infty}$  of Eq. (\ref{rw-infinite-sum}) can be computed analytically as
\begin{equation}\label{eq:inverse}
P_{\infty} \, = \, (1\!-\!\alpha)\big(I - \alpha S\big)^{-1}G
\end{equation}
\end{theoremOne}
\begin{proof}
Consider the partial sum over the $L$ first terms of Eq. (\ref{rw-infinite-sum})
\begin{equation}\label{eq:partial-sum}
\wt{P}_{L} \, := \, (1\!-\!\alpha)\left(\sum_{t=0}^L \alpha^{t} S^{t}\right)\!G.
\end{equation}
Pre-multiplying by $(I-\alpha S)$ gives
\begin{align}
(I-\alpha S) \wt{P}_{L} & \, = \, (1\!-\!\alpha)\left(\sum_{t=0}^L \alpha^{t} S^{t} - \sum_{t=1}^{L+1} \alpha^{t} S^{t}\right)\!G\nonumber\\
& \, = \,  (1\!-\!\alpha)\left(I - (\alpha S)^{L+1}\right)G.
\end{align}
Since $\alpha\!<\!1$, we have that $(\alpha S)^{L+1} \!\to\! 0$ when $L \!\to\! \infty$, hence
\begin{align}
(I-\alpha S) & P_{\infty} \, = \, (1\!-\!\alpha)G \nonumber\\
& P_{\infty} \, = \, (1\!-\!\alpha)\big(I - \alpha S\big)^{-1}G.
\end{align}
\end{proof}
Nevertheless, inverting a large ($N\!\times\!N$) and potentially dense matrix during training and inference is impractical. In the next sections, we show how the random walk probabilities can be computed efficiently using the low-rankness and sparseness of $S$. 

\subsection{Node-to-Node Transition Matrix}\label{sec:node-to-node}

\mypar{Global Affinity Matrix.} Building on recent OVSS methods \cite{ovdiff,diffseg,iseg}, we leverage the attention mechanism of a pre-trained diffusion model to capture global dependencies within the image. Specifically, we use the last two self-attention maps of the Stable Diffusion model \cite{stablediffusion}, which employs a U-Net architecture for its denoising network. Given a feature map with spatial dimensions $H\!\times\!W$ and channel dimension $D$, where each spatial location is treated as a token. The input feature map is first linearly projected into query $Q$, key $K$, and value $V$ representations. Attention scores are then calculated as the scaled dot-product between queries and keys, followed by a softmax normalization:
\begin{equation}
A_{\text{self}} \, := \, \text{softmax}\!\left(\frac{Q\tr{K}}{\sqrt{D}}\right),
\end{equation}
Finally, the output is obtained as $Y\!=\!A_{\text{self}}\, V$. 
Due to the softmax operation, $A_{\text{self}}$ must be constructed explicitly. This leads to a computational complexity that scales quadratically with the number of tokens $N$, posing a challenge for segmentation tasks that require high-resolution inputs. 

To solve this problem, we draw inspiration from linear attention transformers \cite{katharopoulos2020transformers} and discard the softmax function to compute our global affinities matrix as
\begin{equation}\label{eq:single-head}
[A_{\text{global}}]_{i,j} \, := \, \frac{\langle Q_i,\;K_j\rangle}{\|Q_i\|\!\cdot\!\|K_j\|},
\end{equation}
with $\langle\cdot,\cdot\rangle$ denoting the dot product.
%As described in Section \ref{sec:rw-ovss}, since this matrix is low-rank when $D\!<\! N$, it enables a fast computation of the random-walk probabilities.

\mypar{Local Affinity Matrix.} To model local interactions, we define for each token $i$ the set of its immediate spatial neighbors using an 8-connected neighborhood:
\begin{align*}
\mathcal{N}(i) \, := \, \big\{ j \;:\; &\; \mathbf{x}_j = \mathbf{x}_i + (dh, dw), \\[-1pt]
&\; (dh, dw) \in \{-1, 0, 1\}^2 \setminus \{(0, 0)\} \big\}.
\end{align*}
This ensures that each token exchanges information with its spatially adjacent tokens. We then define our \emph{sparse} local affinity matrix as
\begin{equation}
[A_{\text{local}}]_{i,j} \, := \, 
  \begin{cases}
    \epsilon_{\text{self}}, & i=j,\\
    \frac{\langle Q_i,\;K_j\rangle}{\|Q_i\|\cdot\|K_j\|}, & j\in\mathcal{N}(i),\\
    0,       & \text{otherwise}.
  \end{cases}
\end{equation}
Here, $\epsilon_{\text{self}} > 0$ is a small constant enabling self-transitions in the random walk.

\mypar{Transition Probability Matrix.} To unify the semantic priors captured in the global affinities with the spatial consistency modeled in the local affinities, we apply a linear combination of the two matrices. Toward this goal, we row-normalize them as in Eq. (\ref{eq:row-norm}) to obtain the stochastic matrices $S_{\text{global}}$ and $S_{\text{local}}$. We then compute the transition matrix as 
\begin{equation}\label{eq:label-gen}
S \, := \, \beta S_{\text{global}}
+ (1\!-\!\beta) S_{\text{local}},
\end{equation}
where $\beta \in [0, 1]$ controls the balance between global and local information. This fused affinity serves as the basis for the final segmentation mask.

\subsection{Node-to-Label Generation Matrix}

We leverage CLIP's zero-shot capabilities to define the node-to-label generation matrix $G$. Specifically, we take the projected query embeddings $Q$ of tokens in the last CLIP layer and the projected key embeddings $K$ of class prompts, and compute the cross-attention map as
\begin{equation}
A_{\text{cross}} \, := \, \text{softmax}\!\left(\frac{Q\tr{K}}{\sqrt{D}}\right).
\end{equation}
As before, we convert $A_{\text{cross}}$ to stochastic matrix $G$ by normalizing its rows.

\subsection{Uncertainty-weighted Affinities} 

The global affinity matrix in Eq. (\ref{eq:single-head}) corresponds to a single attention head of the ViT. In our approach, we leverage self-attention maps from 10 different heads: 5 from the final layer and 5 from the penultimate one. A common practice for aggregating these maps is to apply a fixed-weight average, often assigning higher weights to maps from deeper layers \cite{iseg,naclip,ding2023open,clearclip,sclip,proxyclip,khani2023slime,sun2024cliper,diffseg}. However, such fixed strategies fail to capture the image-specific importance of each attention map. To overcome this limitation, we propose a dynamic weighting scheme using entropy-based uncertainty. 

Let $A^{(h)} \!\in\! \mathbb{R}^{N \times N}$ be the attention map from head $h$ in any layer, and denote as $S^{(h)}$ the stochastic matrix obtained by row-normalizing $A^{(h)}$. A single-step random-walk using $S^{(h)}$ would produce a label probability matrix
\begin{equation}
P^{(h)}_{\text{1-step}} \, := \, S^{(h)} G.
\end{equation}
If $A^{(h)}$ models semantically-relevant relations between different parts of the image, we expect the segmentation to be confident. Hence, we can use the entropy of $P^{(h)}_{\text{1-step}}$ to estimate the usefulness of attention map $A^{(h)}$. Specifically, we compute the entropy of $A^{(h)}$ as
\begin{equation}
\mathcal{H}^{(h)} \, := \, -\frac{1}{N} \sum_{i=1}^{N} \sum_{k=1}^{K} \big[P^{(h)}_{\text{1-step}}\big]_{i,k}\log \big[P^{(h)}_{\text{1-step}}\big]_{i,k}.
\end{equation}
We then use this uncertainty value to obtain a weight $w_h$ measuring the usefulness of $A^{(h)}$:
\begin{equation}
w_h \, := \, \frac{\exp(-c\!\cdot\!\mathcal{H}^{(h)})}{\sum_{h'} \exp(-c\!\cdot\!\mathcal{H}^{(h')})},
\end{equation}
where $c$ controls how ``peaked'' the distribution of weights is. Following this, the uncertainty-weighted affinity matrix is obtained as
\begin{equation}
A_{\text{weighted}} \, := \, \sum_{h=1}^{H} w_h\!\cdot\!A^{(h)}.
\end{equation}
To use this weighted-variant of global affinities, we simply replace $S_{\text{global}}$ and $S_{\text{local}}$ by their weighted version in Eq. (\ref{eq:label-gen}).

\subsection{Efficient Random-walk Computation}\label{sec:efficient-rw}

\mypar{Truncated random-walk.} When using a single attention head, matrix $S_{\text{global}}$ has the property of being low rank. As described in Supplementary Materials, this enables to compute 
probabilities $P_{\infty}$ exactly and efficiently using the Woodbury matrix inversion identity. However, this property vanishes when combining the attention matrices of multiple heads. To overcome this problem, we instead adopt a truncated random-walk approach, where the propagation is performed for a fixed number of steps $T$:
\begin{equation}
P_{L} \,:=\, \frac{1\!-\!\alpha}{1\!-\!\alpha^{L+
1}}\left(\sum_{t=0}^{L} \alpha^t S^t\right)G
\end{equation}
Because it truncates the random walk, this strategy may lead to different results than original model of infinite random walk. As shown in the following Theorem, this difference quickly becomes small as we increase the length of the walk.

\begin{theoremTwo}%[Pythagorean theorem]
\label{theorem2}
Let $R_L$ be the sum of truncated terms in a random walk of length $L$:
\begin{equation}
R_L \, := \, (1\!-\!\alpha)\left(\sum_{t=L+1}^{\infty}\!\!\alpha^t S^t\right)G. 
\end{equation}
The L1 norm of $R_L$ goes exponentially fast to $0$ as we increase the number of steps $L$.
\end{theoremTwo}

\begin{proof}
Since both $S$ and $G$ are stochastic, we have that $S^t G$ is also stochastic, hence $\|S^t G\|_1 = N$. Therefore, 
\begin{align}
\|R_L\|_1 & \, = \, (1\!-\!\alpha) \sum_{t=L+1}^{\infty}\!\!\alpha^t \left\|S^t G\right\|_1 \nonumber\\
& \, = \, N \!\cdot\!(1\!-\!\alpha) \sum_{t=L+1}^{\infty}\!\!\alpha^t \nonumber\\
& \, = \, N \!\cdot\!(1\!-\!\alpha)\left(\sum_{t=0}^{\infty}\alpha^t - \sum_{t=0}^{L}\alpha^t \right)\nonumber\\
& \, = \, N \!\cdot\!(1\!-\!\alpha)\left(\frac{1}{1\!-\!\alpha} - \frac{1\!-\!\alpha^{L+1}}{1\!-\!\alpha}\right) \nonumber\\
& \, = \, N \!\cdot\!\alpha^{L+1}.
\end{align}
\end{proof}
For $\alpha\!<\!1$, $\|R_L\|_1$ will hence go to zero exponentially fast (geometric decay) as $L \!\to\! \infty$. 

%%% HERE 
\mypar{Iterative computation.} To compute the truncated random-walk efficiently, we employ an iterative strategy where the label probabilities are updated as follows:
\begin{equation}\label{eq:update}
\wt{P}_{L} \, := \, (1\!-\!\alpha)G + \alpha S\wt{P}_{L-1}, \ \ P_{L} \, := \, \frac{1}{1\!-\!\alpha^{L+1}} \wt{P}_{L} 
\end{equation}
where $\wt{P}_{L}$ is the unnormalized partial sum defined in Eq. (\ref{eq:partial-sum}) and $\wt{P}_{0} = (1\!-\!\alpha)G$. Denote as $S^{(h)} := \beta S_{\text{global}}^{(h)} + (1\!-\!\beta)S_{\text{local}}^{(h)}$ the transition matrix of diffusion model's attention head $h$. Matrix $S^{(h)}_{\text{global}}$ can be expressed in a low rank form as
\begin{equation}
S^{(h)}_{\text{global}} \, := \, \wt{Q}^{(h)}\tr{(K^{(h)})}
\end{equation}
where
\begin{equation}
\mr
{Diag}\big(Q^{(h)}\tr{(K^{(h)})}\One)\big)^{-1}Q^{(h)}
\end{equation}
and $\One$ is a vector with all ones. Expanding Eq. (\ref{eq:update}), we then get
\begin{align}
\wt{P}_{L} & \, := \, (1\!-\!\alpha)G \ + \ \alpha\sum_{h=1}^H w_h \Big(\beta\!\!\underbrace{\wt{Q}^{(h)}\big(\tr{(K^{(h)})} \wt{P}_{L-1}\big)}_{\mathcal{O}(N^2\cdot K \cdot D) \, + \, \mathcal{O}(N \cdot K \cdot D^2)} \nonumber\\[-2pt]
& \qquad \ \ + \ (1\!-\!\beta)\underbrace{S^{(h)}_{\text{local}}\,\wt{P}_{L-1}}_{\mathcal{O}(N\cdot K)}\Big).
\end{align}

For each update step $L$ and every head $h$, computing the \emph{global} label probabilities using the low rank decomposition has complexity in $\mathcal{O}(N^2\!\cdot\!K \!\cdot\! D) \, + \, \mathcal{O}(N \!\cdot\! K \!\cdot\! D^2)$, where $N$ is the number of spatial tokens, $K$ the number of classes and $D$ is the dimension of token features. Likewise, obtaining the \emph{local} probabilities with sparse matrix $S^{(h)}_{\text{local}}$ has a complexity in $\mathcal{O}(N \!\cdot\! K)$. For large images, we have that $D \ll N$ thus the overall complexity simplifies to $\mathcal{O}(N^2 \!\cdot\! K \!\cdot\! D)$. This sharply contrasts with the $\mathcal{O}(N^3 \!\cdot\! K)$ complexity of approaches like iSeg \cite{iseg} that do not leverage the low rank property of attention matrices. The computation of segmentation probabilities is further detailed in the Supplmental Materials.

\section{Experiments}
\label{sec:experiment}

\subsection{
Setup} \label{sec:exp_setup}

\mypar{Datasets.}
We evaluate \ours{} %our method on the benchmark mentioned, the names are abbreviated (in parentheses) too conserve table space:
on five standard OVSS benchmarks:
PASCAL VOC 2012 (V21) \cite{voc12},
ADE20K-150 (ADE) \cite{ade20k},
PASCAL Context (PC60) \cite{pascalcontext},
COCO-Stuff (C-Stf) \cite{coco},
COCO-Object (C-Obj) \cite{actualcoco}. 
Also, alongside the original benchmarks on these datasets, we follow \cite{iseg} and evaluate on variants of PASCAL VOC 2012 (V20) and PASCAL Context (PC59) in which the background class is removed from the evaluation. 

\mypar{Baselines.}
Our method is compared with a relevant set of works in OVSS, including: %open-vocabulary semantic segmentation which includes 
CLIP ~\cite{clip}, MaskCLIP ~\cite{maskclip}, GroupViT ~\cite{wysoczanska2024clip}, GEM ~\cite{bousselham2024grounding}, SCLIP ~\cite{sclip}, NACLIP ~\cite{naclip}, iSeg ~\cite{iseg}, ClearCLIP ~\cite{clearclip}, ProxyCLIP ~\cite{proxyclip}, OVDiff ~\cite{ovdiff}, TagCLIP ~\cite{lin2024tagclip}, CaR ~\cite{car}, and PixelCLIP ~\cite{pixelclip}.

\mypar{Implementation Details.}
For our experiments we employ pretrained CLIP-ViT ~\cite{clip}, with  ViT-L/14 backbone, and Stable Diffusion V2.1~\cite{stablediffusion}. We set the category descriptions as in SCLIP~\cite{sclip}. Furthermore, the input images are resized to 336\,$\times$\,336. We implemented our approach on a single RTX A6000 with 48G and utilizing less than 10G memory, with the inference being done with batch size of 1 using half precision. Last, we resort to mIoU as the evaluation metric across all experiments. 

% We employ the pre-trained stable diffusion V2.1 as the base model and feed the input image into the pre-trained model for feature extraction without any training on the segmentation dataset. We implemented our approach on a single RTX A6000 with 48G and utilizing less than 10G memory. The inference is done with batch size of 1 using half precision.

\begin{table*}[htb]
    \centering
    \caption{
    \textbf{Quantitative evaluation across five datasets and two of their variants.} The first three benchmarks (V21, PC60, and C-Obj) include a background category, whereas the subsequent ones do not. The \textit{Post.} column indicates whether an approach employs post-processing for mask refinement. For explanations of abbreviated benchmark names, refer to \cref{sec:exp_setup}.
    \vspace{-3pt}
    }    
    \resizebox{\linewidth}{!}{
        \begin{tabular}{l@{\hskip 5px}lcllllllllll}
        \toprule
        \textbf{Method}                                                    &&  Post.& V21  & PC60 & C-Obj & V20 & ADE & PC59  & C-Stf & Avg \\
        \midrule
        CLIP~\cite{clip}                    & {\scriptsize ICML'21}        & \xmark & 18.6 &  \phz7.8 &  \phz6.5          & 49.1          &  \phz3.2 & 11.2  &  \phz5.7 & 13.6 \\ 
        MaskCLIP~\cite{maskclip}            & {\scriptsize ECCV'22}        & \xmark & 43.4 & 23.2 & 20.6          & 74.9          & 16.7 & 26.4  & 16.7 & 30.3 \\
        GroupViT~\cite{groupvit}            & {\scriptsize CVPR'22}        & \xmark & 52.3 & 18.7 & 27.5          & 79.7          & 15.3 & 23.4  & 15.3 & 30.7 \\
        CLIP-DIY~\cite{wysoczanska2024clip} & {\scriptsize WACV'24}        & \xmark & 59.0 & -    & 30.4          & -             & -    & -      & -    & -    \\
        GEM~\cite{bousselham2024grounding}  & {\scriptsize CVPR'24}        & \xmark & 46.2 & -    & -             & -             & 15.7    & 32.6  & -    & -    \\
        SCLIP~\cite{sclip}                  & {\scriptsize ECCV'24}        & \xmark & 59.1 & 30.4 & 30.5 & 80.4 & 16.1  & 34.2  & 22.4 & 38.2 \\
        PixelCLIP~\cite{pixelclip}          & {\scriptsize NeurIPS'24}         & \xmark & - & - & -          & \underline{85.9}  & \underline{20.3} & \underline{37.9} & \underline{23.6} & - \\

        NACLIP~\cite{naclip}                & {\scriptsize WACV'25}        & \xmark & 58.9 & \underline{32.2} & 33.2 & 79.9 & 17.4  & 35.2  & 23.3 & \underline{39.4} \\
        iSeg~\cite{iseg}                    & {\scriptsize arXiv'24}          & \xmark & \underline{68.2} & 30.9 & \underline{38.4}          & - & - & - &  - & - \\
        \rcol \textbf{\ours{}}  & {\scriptsize Ours} &  \xmark & \textbf{69.7\better{1.5}} & \textbf{37.7\better{5.5}} & \textbf{43.3\better{4.9}} & \textbf{90.1\better{4.2}} & \textbf{24.0\better{3.7}} & \textbf{43.4\better{5.5}} & \textbf{28.8\better{5.2}} & \textbf{48.1\better{4.3}} \\  
        \midrule
                SCLIP~\cite{sclip}                  & {\scriptsize ECCV'24}       & \cmark & 61.7 & 31.5 & 32.1          & 83.5 & 17.8 & 36.1 & 23.9 & 40.1 \\
        ClearCLIP~\cite{clearclip}          & {\scriptsize ECVA'24}         & \cmark & 46.1 & 26.7 & 30.1          & 80.0 & 15.0 & 29.6  & 19.9 & 35.3 \\
        ProxyCLIP~\cite{proxyclip}          & {\scriptsize ECCV'24}         & \cmark & 60.6 & 34.5 & \underline{39.2}          & 83.2 & \underline{22.6} & 37.7  & 25.6 & \underline{43.3} \\
        OVDiff~\cite{ovdiff}                & {\scriptsize CVPR'24}         & \cmark & -    & -    & -          & 80.9 & 14.1 & 32.2 & 20.3  & - \\
        CaR~\cite{car}                      & {\scriptsize CVPR'24}         & \cmark & \underline{67.6} & 30.5 & 36.6          & \textbf{91.4} & 17.7 & \underline{39.5} & -  & - \\
        NACLIP~\cite{naclip}                & {\scriptsize WACV'25}       & \cmark & 64.1 & \underline{35.0} & 36.2 & 83.0 & 19.1 & 38.4 & \underline{25.7} & 42.5 \\ 
        \rcol \textbf{\ours{}}  & {\scriptsize Ours}  & \xmark & \textbf{69.7\better{2.1}} & \textbf{37.7\better{2.7}} & \textbf{43.3\better{4.1}} & \underline{90.1}\worse{1.3} & \textbf{24.0\better{1.4}} & \textbf{43.4\better{3.9}} & \textbf{28.8\better{3.1}} & \textbf{48.1\better{2.3}} \\  
        \bottomrule
        \end{tabular}
        }
    \label{tab:main}
\end{table*}

\subsection{Main Results}

\mypar{Comparison to training-free OVSS methods.}
We report a comprehensive evaluation of \ours{} %on six standard open-vocabulary segmentation benchmarks and two additional variants, as shown 
in \cref{tab:main},  %These include datasets with both background-inclusive categories (V21, PC60, C-Obj) and background-excluded variants (V20, ADE, PC59, C-Stf). We compare \ours{} against a wide range of prior works, including both \emph{post-processing-free} methods such as CLIP~\cite{clip}, MaskCLIP~\cite{maskclip}, GroupViT~\cite{groupvit}, SCLIP~\cite{sclip}, NACLIP~\cite{naclip}, ClearCLIP~\cite{clearclip}, and iSeg~\cite{iseg}, as well as \emph{post-processing-enhanced} methods such as ProxyCLIP~\cite{proxyclip}, CaR~\cite{car}, and TagCLIP~\cite{tagclip}.
where results show that \textbf{\ours{} consistently outperforms prior %post-processing free 
approaches without post-processing across all datasets}. On V21, our method achieves an mIoU of 69.4\%, surpassing SCLIP (59.1\%) and NACLIP (58.9\%), and also exceeding the recent diffusion-based iSeg~\cite{iseg} (68.2\%). On PC60 and C-Obj, \ours{} obtains 37.6\% and 43.3\% mIoU respectively, outperforming all baselines without post-processing by a substantial margin. Notably, even the best baseline without post-processing, i.e., NACLIP, falls short by over 5\% %as average. %points 
on both PC60 and C-Obj.

Our method also demonstrates strong generalization on benchmarks that exclude background classes, where localization precision is especially critical. On V20, \ours{} achieves an mIoU of 90.1\%, improving over both GroupViT (79.7\%) and ProxyCLIP (83.2\%), and approaching CaR (91.0\%), a method that requires post-processing. %—a post-processed method. 
Similarly, on PC59 and C-Stf, \ours{} attains 24.0\% and 43.4\%, substantially outperforming SCLIP (34.2\%, 22.4\%) and NACLIP (35.2\%, 23.3\%). Importantly, many of the top-performing methods in recent literature rely on post-processing strategies such as Conditional Random Fields (CRF) or Pixel-Adaptive Mask Refinement (PAMR), as indicated by the (\cmark{}) in Tab~\ref{tab:main}. For example, ProxyCLIP, CaR, and ClearCLIP all incorporate such steps to increase performance. In contrast, \ours{} achieves state-of-the-art results \emph{without any} post-processing (\xmark), relying solely on a single forward pass with our entropy-guided attention fusion and local–global affinity refinement.

Overall, \ours{} sets a novel performance benchmark among training-free approaches, with an average mIoU of \textbf{48.1\%} across all datasets. These results underscore the effectiveness of our design in producing high-quality, spatially consistent segmentations, without relying on additional supervision, or post-hoc refinement.

\subsection{Ablations Study}
%Here we perform the ablation study to demonstrate the effectiveness of our proposed modules. 

\mypar{Comparison with Weakly‐Supervised Semantic Segmentation (WSSS).} To demonstrate the quality of the pseudo‐masks generated by \ours{}, we compare it against several state‑of‑the‑art WSSS methods. Tab.~\ref{tab:wsss} reports mIoU on the PASCAL VOC 2012 and COCO training sets\footnote{Note that “Training” indicates whether a method requires any per‑dataset learning to generate pseudo masks.}. Our method %(\ours{})—built atop the iSeg backbone—
yields the highest mIoU on both benchmarks, improving over prior %by roughly 1 point over the strongest 
training‑free approaches. %competitor. 
This underscores the benefit of our entropy‐guided fusion and random‑walk refinement, which inject spatial coherence without any additional optimization.

\begin{table}[t]
\centering
\footnotesize
\setlength{\tabcolsep}{4.5pt}%2.0
    \caption{
    \textbf{Comparison of pseudo‐mask generation methods for weakly‐supervised semantic segmentation.}
    We report mIoU on PASCAL VOC 2012 and MS-COCO training sets. %Our method outperforms both training‐based and training‐free approaches.
    %For abbreviated benchmark names, see \cref{sec:exp_setup}.    
    \vspace{-3pt}
    }    
    \resizebox{\columnwidth}{!}{
        \begin{tabular}{l|cccc}
        \toprule
        \textbf{Method} &  & Traning & VOC & COCO \\
        \midrule
        MCTformer+ ~\cite{mct}                    & {\scriptsize PAMI'24}        & \cmark & 68.8 &  - \\ 
        ToCo ~\cite{toco}            & {\scriptsize CVPR'23}        & \cmark & 72.2 & - \\
        WeakTr~\cite{weaktr}            & {\scriptsize CVPR'23}        & \cmark & 66.2 & - \\
        CLIMS~\cite{clims} & {\scriptsize CVPR'22}        & \cmark & 56.6 & -    \\
        CLIP-ES~\cite{clipes}  & {\scriptsize CVPR'23}        & \xmark & 70.8 & 39.7    \\
        DiffSegmenter~\cite{diffseg}                  & {\scriptsize PAMI'25}       & \xmark & 70.5 & - \\
        T2M~\cite{t2m}                  & {\scriptsize NeuroC'24}       & \xmark & 72.7 & 43.7 \\
        iSeg~\cite{iseg}                  & {\scriptsize arXiv'24}       & \xmark & 75.2 & 45.5 \\
        \rcol \textbf{\ours{}}\,(for iSeg) & {\scriptsize Ours} &  \xmark & \textbf{76.2} & \textbf{45.7}\\  
        \bottomrule
\end{tabular}
}
\vspace{-6pt}
\label{tab:wsss}
\end{table}

\mypar{Multi‑Head Attention Fusion in Stable Diffusion (Tab.~\ref{tab:ablation_attention_fusion}).}
To isolate the effect of our entropy‑guided fusion mechanism on the multi‑head attention maps produced by the Stable Diffusion backbone, we compare three aggregation strategies: (1) \emph{Single}, which selects the best‐performing head; (2) \emph{Mean}, which averages all heads uniformly; and (3) \emph{Weighted Mean}, which fuses heads using entropy‑based weights that amplify low‑entropy (i.e., more focused) heads. %Table~\ref{tab:ablation_attention_fusion} reports mIoU on the VOC, Pascal Context, and Object segmentation benchmarks. 
As the results show, \emph{Weighted Mean} consistently outperforms both \emph{Single} and \emph{Mean}, achieving 69.4\% on VOC (a +2.5 improvement over Single and +2.8 over Mean), 37.6\% on Context (+1.0 over Single and +1.1 over Mean), and 43.3\% on Object (+1.3 over Single and +1.6 over Mean). These gains confirm that entropy‑guided weighting effectively suppresses noisy or diffuse heads while emphasizing the most informative ones, yielding more accurate segmentation masks. Consequently, we adopt \emph{Weighted Mean} as our default fusion strategy in all subsequent experiments.

\begin{table}[h!]
\centering
\footnotesize
\setlength{\tabcolsep}{12pt}%2.0
  \caption{%
    \textbf{Comparison of multi‑head attention fusion strategies in Stable Diffusion.}
    “Single” selects the best head, “Mean” averages all heads, and “Weighted Mean” uses entropy‑based head weighting.
  }\label{tab:ablation_attention_fusion}
\vspace{-3pt}
\resizebox{\columnwidth}{!}
{
\begin{tabular}{l|ccc}
\toprule
Type        & VOC                      & Context       & Object        \\ 
\midrule
Single      & 67.2                     & 36.6          & 42.0          \\
Mean        & 66.9                  & 36.5          & 41.7         \\
\rcol
Weighted Mean
        & \textbf{69.7} 
        & \textbf{37.6} & \textbf{43.3} \\ \bottomrule
\end{tabular}
}
\vspace{-6pt}
%\label{tab:abltation_refinement}
\end{table}

\begin{table}[t]
\centering
\footnotesize
\setlength{\tabcolsep}{9.5pt}%2.0
\caption{\textbf{Ablation on different affinity maps.} Global represents the affinity map coming directly from the self attention layer. Local represents the affinity constructed using cosine similarity explained above, whereas RW represents the stochastic random walk process used for mask refinement.
\vspace{-3pt}
}
\resizebox{\columnwidth}{!}
{
\begin{tabular}{l|ccc}
\toprule
Type        & VOC                      & Context       & Object        \\ 
\midrule
Global      & 63.3                     & 35.0          & 39.8          \\
Local        & 58.7                  & 34.5          & 37.6         \\
\rcol
Global + Local + RW
        & \textbf{69.7} 
        & \textbf{37.6} & \textbf{43.4} \\ \bottomrule
\end{tabular}
}
\vspace{-6pt}
\label{tab:local_global}
\end{table}

\begin{figure*}[ht!]
  \centering
  \begin{minipage}[b][][b]{.38\textwidth}
  % span both columns and fill the full text width
  \includegraphics[width=\linewidth]{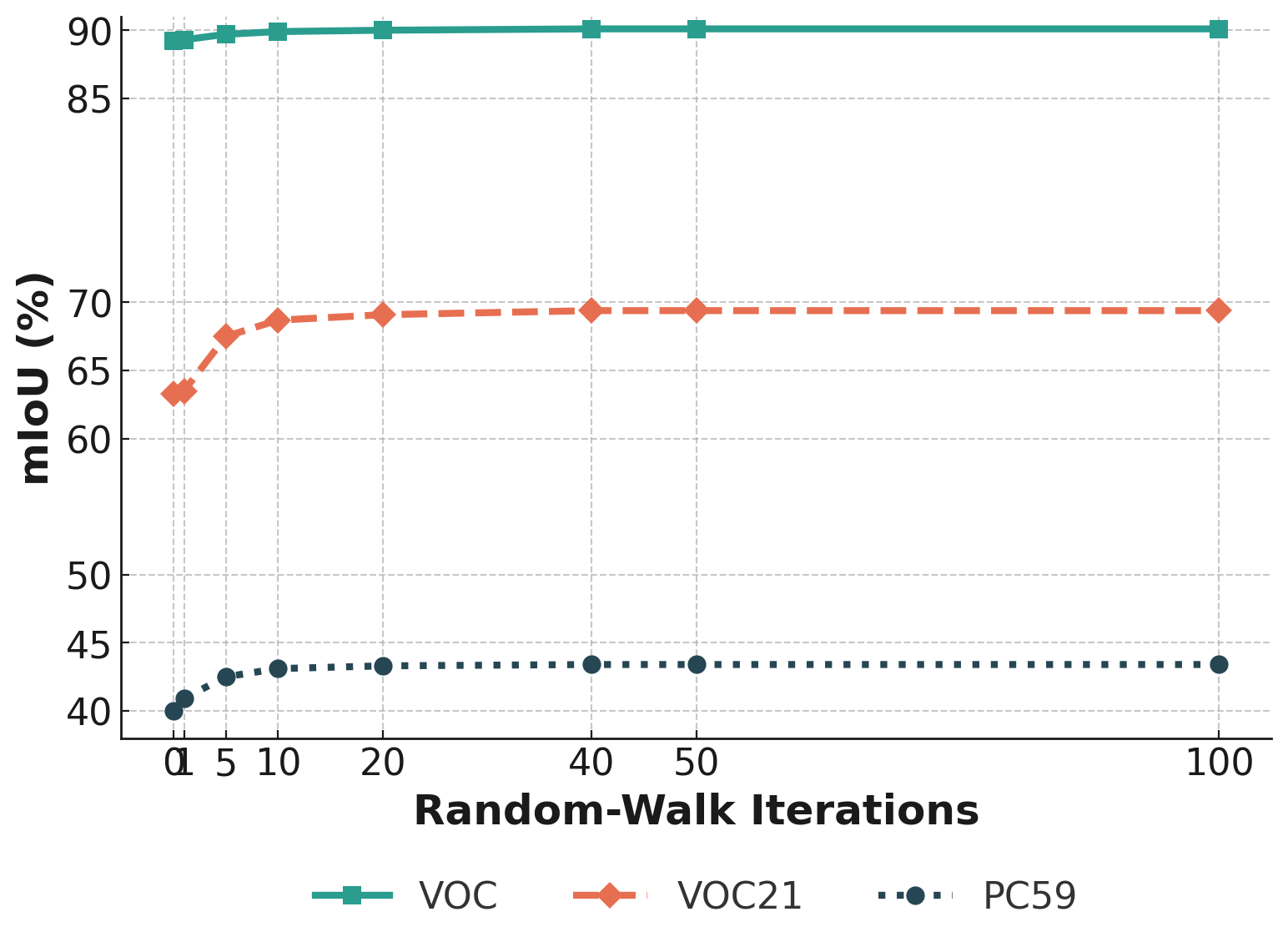}
\caption{%
  \textbf{Impact of random‑walk iteration count on segmentation mIoU.}
  Evolution of mIoU on PASCAL VOC, VOC21, and PC59  as a function of the number of stochastic random‑walk refinement steps, demonstrating rapid gains up to 20 iterations and saturation thereafter.
}
  \label{fig:miou_plot}
  \end{minipage}
  \hfill
  \begin{minipage}[b][][b]{.595\textwidth}
  \includegraphics[width=\linewidth]{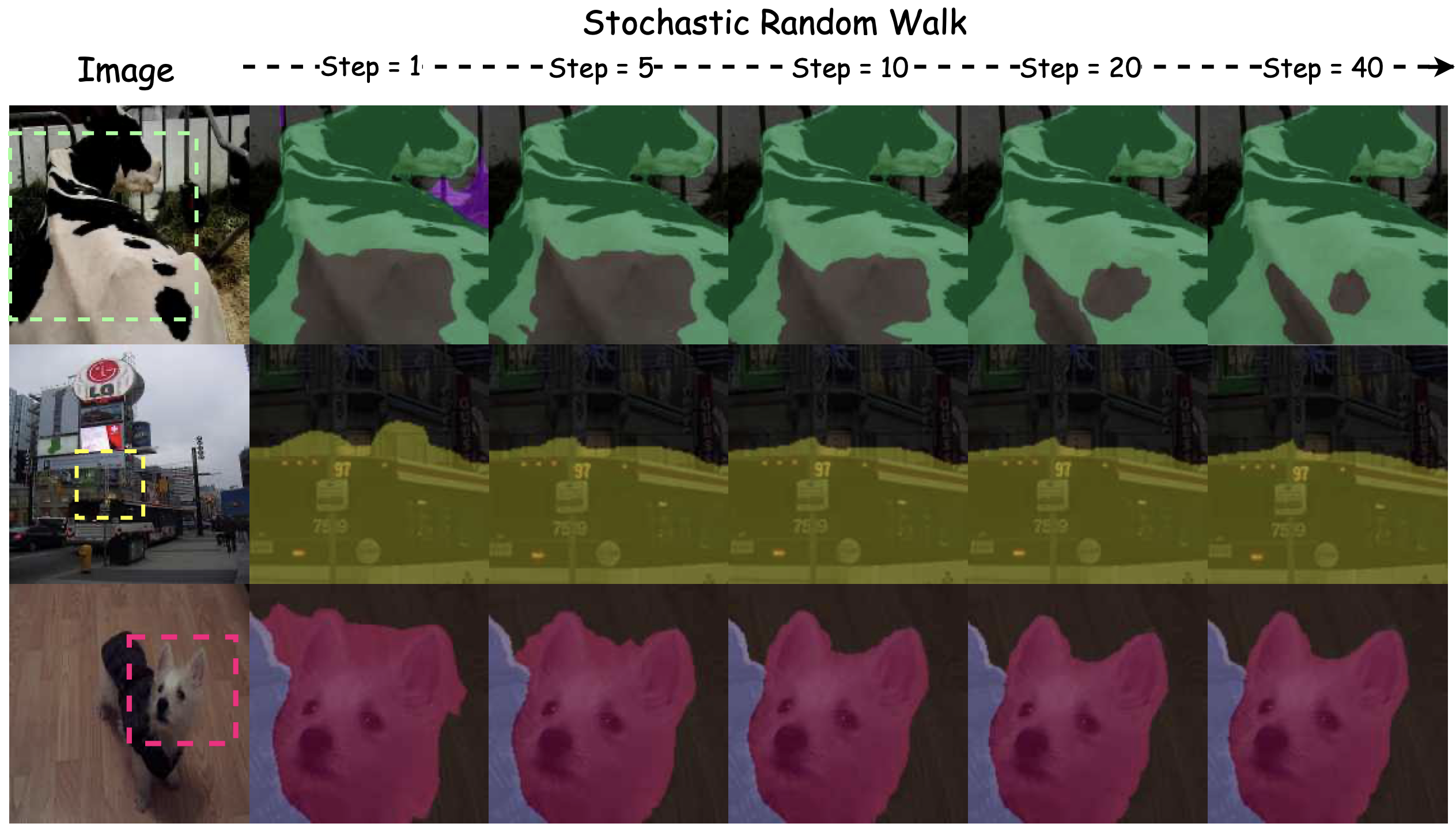}
    \caption{
    \textbf{Progressive segmentation refinement via Stochastic Random Walk.}
    For each image, we visualize the effect of increasing random walk steps on the segmentation map. Regions of interest are highlighted with colored boxes. Our method produces increasingly accurate and coherent segmentations over time, illustrating strong local–global propagation.}
    \label{fig:qualitative_rw}
    \label{fig:qualitative}
  \end{minipage}
\end{figure*}

\mypar{Affinity Map Variants (Table~\ref{tab:local_global}).}
We evaluate three types of affinity maps: (1) \emph{Global}, derived directly from raw self-attention affinity; (2) \emph{Local}, built via cosine similarity between neighboring tokens; and (3) \emph{Global+Local+RW}, which refines the combined affinity through a stochastic random-walk (RW). As shown in Table \ref{tab:local_global}, the \emph{Global+Local+RW} configuration consistently outperforms others (e.g., achieving 69.7\% on VOC). This confirms that integrating global semantic cues, local spatial coherence, and RW-based refinement significantly boosts segmentation accuracy.

\mypar{Effect of Random‑Walk Iteration Count on Segmentation mIoU (Fig.~\ref{fig:miou_plot}).}
As shown in Fig.~\ref{fig:miou_plot}, segmentation performance steadily improves as the number of random‑walk iterations increases. Starting from the 0-step baseline (89.2\% in VOC, 63.3\% in VOC21, and 40.0\% in PC59), 5 random-walk iterations raise mIoU to 89.7\%, 67.5\%, and 42.5\%, respectively. At 10 steps, performance increases to 89.9\%, 68.7\%, and 43.1\%. Refinement with 40 iterations yields mIoU values of  90.1\%, 69.4\%, and 43.4\% on the three benchmarks. Beyond 40 steps, results plateau, indicating convergence.

\subsection{Qualitative Results}

To explore the effect of stochastic random walk, we visualize its progression %the progression of the random walk 
across different steps in Fig.~\ref{fig:qualitative}. We highlight and zoom in on a small region (marked by colored bounding boxes) to show the segmentation output at steps: 1, 5, 10, 20, 40. We notice that in the early steps (i.e., Steps 1 to 5), it can propagate the semantic structure to its neighbors. As the steps increase, the segmentation boundaries become sharper. Notably, at Step 40, the segmentation boundaries are well aligned with visual cues. This improvement is evident across various objects and scenes. In the first row, the green mask of the cow becomes more complete as the walk progresses, while in the second row, even finer details on the bus are filled in by the yellow mask despite the cluttered background. Finally, in the last row, the pink mask on the dog's face noticeably becomes more accurate to the edges. This demonstrates our method's robustness across all objects and contexts. 

%These results highlight the core advantage of our approach: effective semantic propagation from sparse activation maps using an interpretable and lightweight refinement process, without requiring any learned parameters or post-processing.

\section{Conclusion}
\label{sec:conclusion}

Utilizing the zero-capability of CLIP and Stable Diffusion in the OVSS %open-vocabulary semantic segmentation 
domain, this method achieves a significant performance boost and is emerging as an alternative to closed-set supervised training. Our method \textbf{\ours{}} has three main components: an entropy-based weighted mean of the global and local affinity map, and a stochastic random walk for affinity map refinement. The local-global affinity aims to reduce the irrelevant global information typically generated by the self-attention map, and the stochastic random walk, in combination, pays more attention to the features of the segmented class, making this method preferable for semantic segmentation tasks. Extensive experiments on the well-known and recent OVSS benchmark show the superiority of the existing OVSS methods. The empirical observations highlight that our approach yields state-of-the-art performance on post-processing-free methods, and even with methods that require post-processing, our method outperforms 6 out of 7 benchmarks. \ours{}, does not require access to any labelled or unlabelled data, which makes it suitable for real-world applications. 

%With the proposed methodology of the OVSS, it can generate precise segmentation masks without any training of the segmentation datasets. We have demonstrated them with the experiments performed. 
{
    \small
    \bibliographystyle{ieeenat_fullname}
    \bibliography{main}
}

\end{document}